\documentclass[nohyperref]{article}

\usepackage{microtype}
\usepackage{graphicx}
\usepackage{subfig}
\usepackage{booktabs} 
\graphicspath{{figs/}}
\usepackage{hyperref}

\usepackage[accepted]{icmlForArxiv}

\usepackage{amsmath}
\usepackage{amssymb}
\usepackage{mathtools}
\usepackage{amsthm}
\usepackage[capitalize,noabbrev]{cleveref}
\theoremstyle{plain}
\newtheorem{theorem}{Theorem}[section]

\newtheorem{lemma}[theorem]{Lemma}

\theoremstyle{definition}
\newtheorem{definition}[theorem]{Definition}
\newtheorem{assumption}[theorem]{Assumption}
\theoremstyle{remark}

\newcommand{\eq}[1]{\begin{align}#1\end{align}}
\usepackage[disable,textsize=tiny]{todonotes}

\icmltitlerunning{The Learning phases in NN}

\begin{document}

\twocolumn[
\icmltitle{The learning phases in NN: From Fitting the Majority to Fitting a Few}

\icmlsetsymbol{equal}{*}

\begin{icmlauthorlist}
\icmlauthor{Johannes Schneider}{equal,yyy}

\end{icmlauthorlist}

\icmlaffiliation{yyy}{Institute of Information Systems, University of Liechtenstein, Vaduz, Liechtenstein}

\icmlcorrespondingauthor{Johannes Schneider}{johannes.schneider@uni.li}

\icmlkeywords{deep learning, learning phases, compression, fitting}

\vskip 0.3in
]

\printAffiliationsAndNotice{}  

\begin{abstract}
The learning dynamics of deep neural networks are subject to controversy. Using the information bottleneck (IB) theory separate fitting and compression phases have been put forward but have since been heavily debated. We approach learning dynamics by analyzing a layer's reconstruction ability of the input and prediction performance based on the evolution of parameters during training. We show that a prototyping phase decreasing reconstruction loss initially, followed by reducing classification loss of a few samples, which increases reconstruction loss, exists under mild assumptions on the data.  Aside from providing a mathematical analysis of single layer classification networks, we also assess the behavior using common datasets and architectures from computer vision such as ResNet and VGG.
\end{abstract}




\section{Introduction}
Deep neural networks are arguably the key driver of the current boom in artificial intelligence both in academia and industry. They achieve superior performance in a variety of domains. Still, they suffer from poor understanding, which has even led to an entire branch of research, i.e., XAI\cite{mesk22}, and to widespread debates on trust in AI within society. Thus, enhancing our understanding of how deep neural networks work is arguably one of key problems in ongoing machine learning research\cite{pog20}. Unfortunately, the relatively few theoretical findings and reasonings are often subject to rich controversy. 

One debate surrounds the core of machine learning: learning behavior. Tishby and Zaslavsky \cite{tish15} leveraged the information bottleneck(IB) framework to analyze learning dynamics of neural networks. IB relies on measuring mutual information between activations of a hidden layer and the input as well as the output. A key qualitative finding was the existence of a fitting and compression phase during the training process. The information theoretic compression is conjectured a reason for good generalization performance. It is frequently discussed in the literature\cite{gei21,jak19}. For once, Tishby et al.'s findings can be considered breakthrough results in the understanding of deep neural networks. Still, they have also been subject to an extensive amount of criticism related to the validity of their findings, i.e., Saxe et al. \cite{saxe19} claimed that Tishby's claims do not generalize to common activation functions. Today, the debate is still ongoing \cite{lor21}. A key challenge is the difficulty in approximating the IB making rigorous mathematical and even empirical analysis difficult.\\

In this work, we also aim to study the learning behavior with a focus on fitting and compression capability of layers but propose a different lens for investigation. We aim to perform a rigorous analysis of a simple scenario that can be generalized rather than relying only on general statements that lack mathematical proof. 
Second, we utilize different measures from IB. Rather than measuring the information a layer provides on an output\cite{tish15}, we measure how well the layer can be used to classify samples, when a simple classifier is trained on the layer activations, i.e., we investigate linear separability of classes given layer activations.  
Rather than measuring the information of a layer with respect to the input\cite{tish15}, we measure the reconstruction error of the input given the layer by utilizing a decoder.
For single layer networks, we show that the reconstruction error is likely to decrease initially, since the network essentially learns the class average, often resembling a prototype. The error can increase later during training, if a classifier shifts from relying of non-noisy features to more noisy features to improve the loss of a few poorly classified samples. We show empirically that such a behavior is common for multiple classifiers and datasets and layers.\\
We first conduct an empirical analysis followed by a theoretical analysis, related work and conclusions.

\begin{figure}[!htb]
\centering{\centerline{\includegraphics[width=0.45\textwidth]{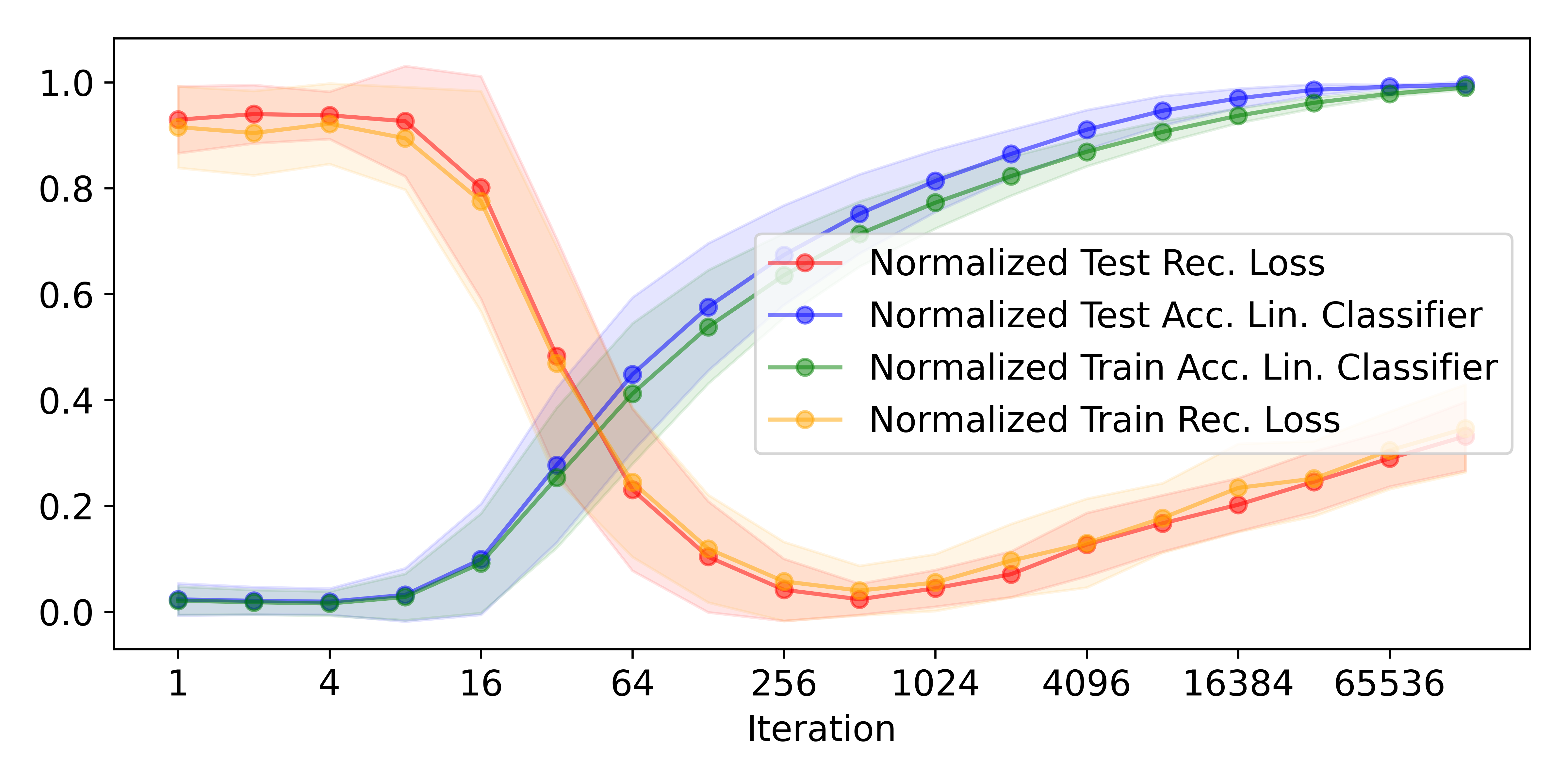}} 
\caption{Normalized accuracy and reconstruction loss for the F0 classifier and the FashionMNIST dataset} \label{fig:metF0Fa}}
\end{figure}

\begin{figure}[!htb]
\centering{\centerline{\includegraphics[width=0.35\textwidth]{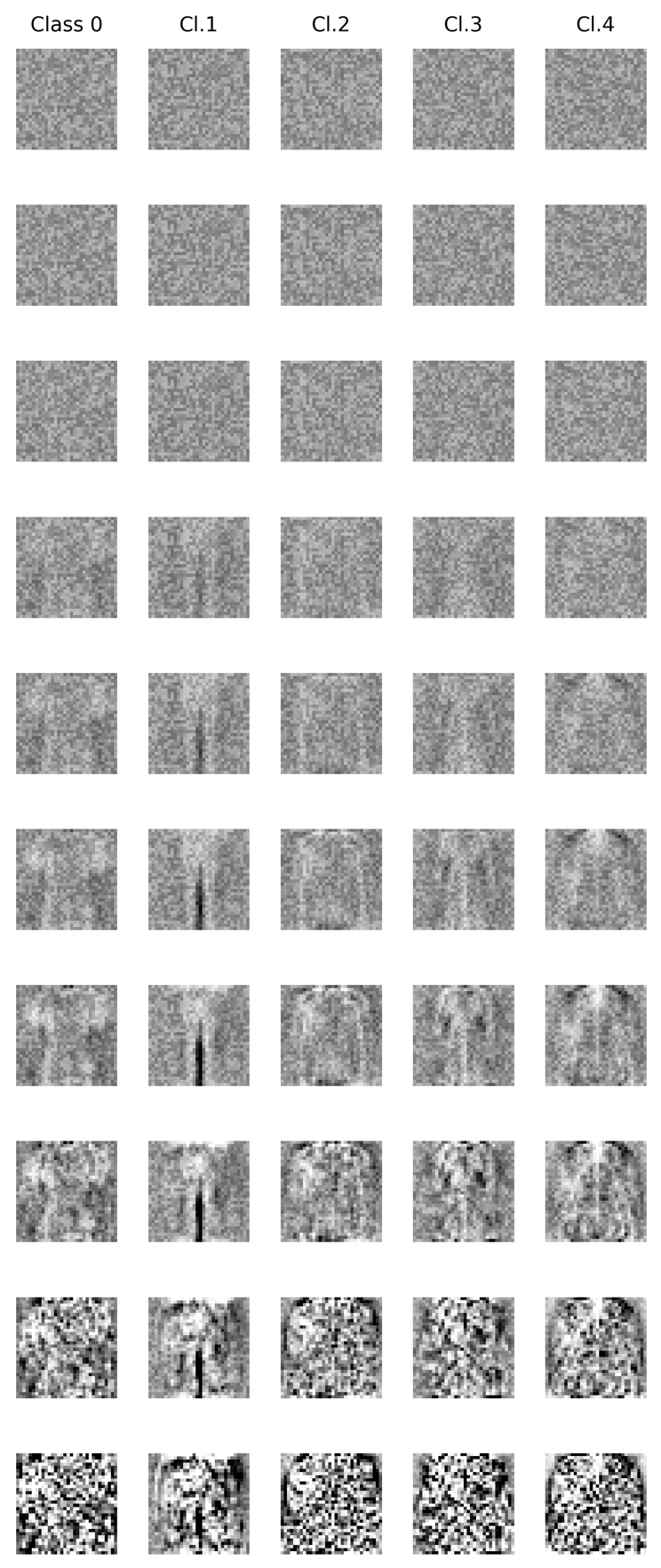}} 
\vspace{-12pt}
\caption{Weight matrices for the F0 classifier for each class across training\footnote{We showed only every 2nd to save space} }
\label{fig:weiFa}}
\vspace{-12pt}
\end{figure}

\section{Empirical analysis}
For a model $M=(L_0,L_1,...)$ consisting of a sequence of layers $L_i$, we discuss behavior of train and test accuracy for a linear classifier $CL$ trained on layer activations $L_i$ of model $M^{(t)}$ at different iterations $t$ during training. We also investigate the reconstruction loss of inputs using a decoder $DE$ to reconstruct inputs $X$ from layer activations $L_i$. 

\subsection{Measures}
We assess outputs of each layer with respect to their ability to predict the output and reconstruct the input. Intuitively, this relates to prior work\cite{tish15} that aimed to capture the amount of information on the input and the output for a given layer.
To compute our measures for a model $M^{(t)}=(L_0,L_1,...)$ trained for $t$ iterations, we train two auxiliary models, i.e. a classifier $CL$ and a decoder $DE$. To assess prediction capability $Acc^{(t)}$ at iteration $t$ of the training, we use a simple dense layer as classifier $CL$ taking as input the outputs $L(X)$ of a layer $L \in M$. In this way, we assess to what extent the outputs $L(X)$ allow us to predict the correct class without much further transformation. 

To obtain the reconstruction error $Rec^{(t)}$, we use a decoder $DE$ that takes as input the outputs $L(X)$ of a layer $L \in M$ and computes the estimate $\hat{X}$ yielding an error $||\hat{X}-X||^2$. 

Both auxiliary models are trained on $L(X)$ for all training data. The metrics $Acc^{(t)}, Rec^{(t)}$ are computed on the test data.

\begin{figure*}[!htb]
\centering{ \centerline{\includegraphics[width=0.9\textwidth]{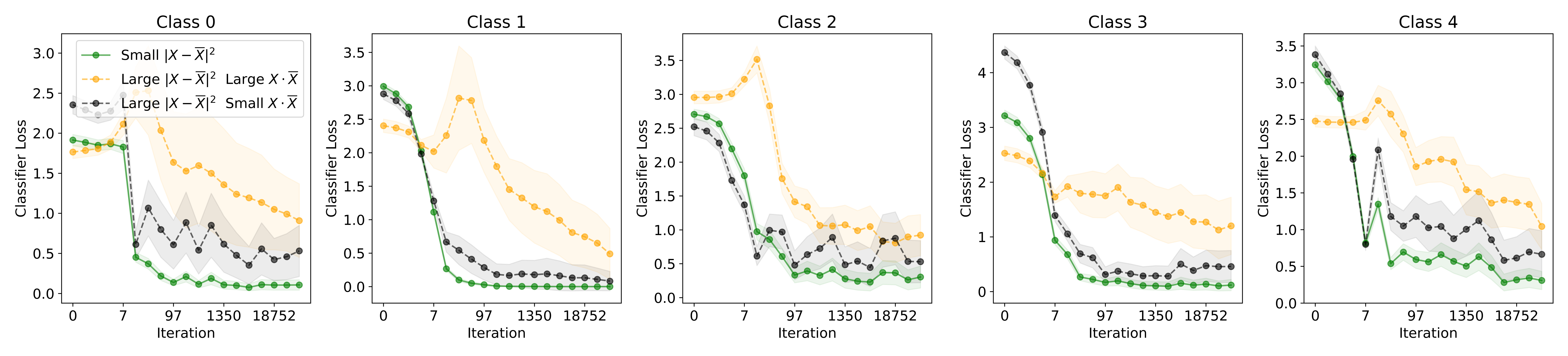}} 
\vspace{-12pt}
\caption{Loss for 1000 samples closest to the mean, samples furthest from the mean with smallest and largest dot product with the mean for FashionMNIST} \label{fig:loClFa} }
\vspace{-12pt}
\end{figure*}

\subsection{Dataset, networks and setup}
As networks for model $M$ we used VGG-11\cite{sim14}, Resnet-10\cite{he16} and fully connected networks, i.e., we employed networks $F0$ and $F1$, where the number denotes the number of hidden layers. A hidden layer has 256 neurons. After each hidden layer we applied the ReLU activation and batch-normalization. We used a fixed learning rate of 0.002 and stochastic gradient descent with batches of size 128 training for 256 epochs.  

We computed evaluation metrics $Acc^{(t)}$, $Rec^{(t)}$ at iterations $2^i$, i.e. $0,1,2,4,8...$.
For the decoder $DE$ we used the same decoder architecture as in \cite{sch21cla}, where a decoder from a (standard) auto-encoder was used.  For each computation of the metrics, we trained the decoder for 30 epochs using the Adam optimizer with learning rate of 0.0003. For the classifier $CL$ we used a single dense layer trained using SGD with fixed learning rate of 0.003 for 20 epochs.
We used CIFAR-10/100\cite{kri09}, Fashion-MNIST\cite{xia17} and MNIST, all scaled to 32x32. We trained each model $M$ 5 times. All figures show standard deviations. We report normalized metrics to better compare $Acc^{(t)}$ and $Rec^{(t)}$.

\subsection{Results}

\begin{figure*}[!htb]
\centering{
\centerline{\includegraphics[width=0.8\textwidth]{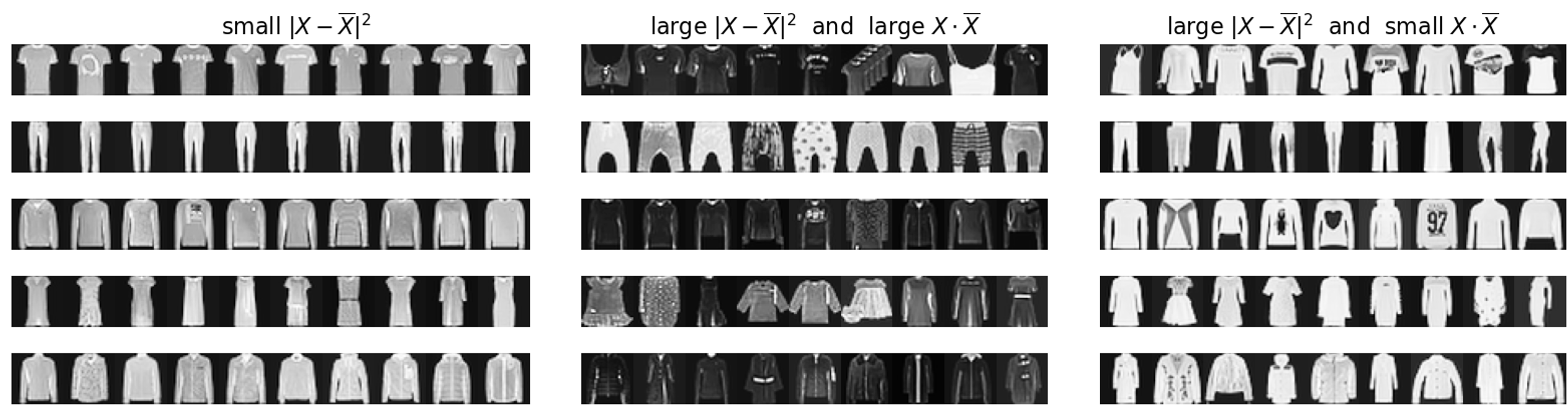}} 
\vspace{-12pt}
\caption{Samples from FashionMNIST close to mean (left), far from mean with low dot product with mean (middle) and with large dot product (right)} \label{fig:loClFaSa}}
\vspace{-12pt}
\end{figure*}

\paragraph{Single Layer networks:} We discuss results for FashionMNIST when using classifier $F0$ as model $M$ to analyze. Figure \ref{fig:metF0Fa} shows as expected that accuracy constantly increases throughout training.  Reconstruction loss remains stable for the first few iterations before decreasing and increasing towards the end. In the light of the information bottleneck theory it was interpreted as the network performing some form of fitting first (leading to a lower reconstruction loss) before compressing, leading to a higher reconstruction loss. We proclaim that first the network moves towards an average of all samples, which can resemble a prototypical class instance. The process of learning an average is well visible in the weight matrices in Figure \ref{fig:weiFa}. They change from randomly initialized matrices shown in the top row towards well-recognizable objects, e.g., the first column resembles a T-shirt and the second a pant, before worsening. Qualitatively this behavior is shown reconstructions of $DE$ (see Figure \ref{fig:rec}). Towards the end of the training, the ``prototypes'' become less recognizable. Thus, visual recognizability of the weight matrices is aligned with reconstruction loss behavior shown in Figure \ref{fig:metF0Fa}.

\begin{figure}[!htb]
	\centering{
		\centerline{\includegraphics[width=0.45\textwidth]{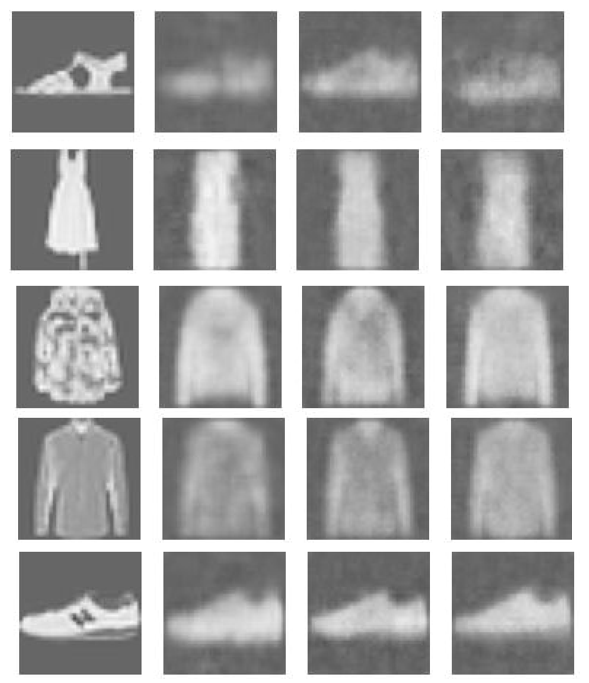}} 
		\caption{Reconstruction of samples (left column) from a decoder $DE$ trained on activations of a model $F0$ being untrained (2nd column), trained for 8 epochs (3rd column) and 256 epochs (4th column) on FashionMNIST. Reconstructions in the third column are better than those of others, which can be noted by careful comparison. Images from the first column seem to be blurrier while those in last column tend to miss the correct grey tone (as seen for jackets) and the front part of the sohle of the shoe in the bottom. They can also show inaccuraccies in shape as shown for the dress and the sandal in the top row. A detailed discussion of the interpretation of layer decodings can also be found in \cite{sch21cla}.} \label{fig:rec}}
\end{figure}

Once samples being similar to the average have low loss, weights are adjusted to correctly classify the remaining samples still having large loss. Conceptually, one can distinguish three types of inputs $X$ based on the dot product $X\cdot \overline{X}$ of a sample and the (class) average $\overline{X}$. Samples $D_S$ that have small dot product, samples $D_A$ with ``average'' dot product and samples $D_B$ with large dot product. Samples $D_A$ are samples similar to the average, i.e. $||X-\overline{X}||^2$ is relatively small for $X \in D_A$, while the difference $||X-\overline{X}||^2$ is large for $X \in (D_B \cup D_S)$. The conceptualization is supported by Figure \ref{fig:diFa} showing the distribution of dot products. Most samples are around the mode, which we denote as $D_A$ those significantly to the left and right correspond to sets $D_S$ and $d_B$. They are smaller in number.

Figure \ref{fig:loClFaSa} shows samples $D_A$ (left column) and $D_B$ (right column), which appear brighter and have similar shape to the mean. Samples $D_S$ (middle column) are darker and exhibit high loss. The loss behavior over time is shown in Figure \ref{fig:loClFa} for 1000 samples from these three sets of samples $D_A, D_B$ and $D_S$ for multiple classes. It is apparent that samples $D_S$ exhibit largest classifier loss, while those with large dot product $D_B$ exhibit lower loss. Interestingly, samples in $D_S$ also commonly show an increase of loss initially. This is expected for samples with negative dot product with the mean. Since weights move towards the average initially, small dot product implies small outputs for these samples for the correct class and in turn low probabilities.
Intuitively, one might expect that samples $D_B$ should be classified even better than those close to the average, i.e., $D_A$. However, by looking at the samples in Figure \ref{fig:loClFaSa}, it becomes apparent that a large dot product still allows for many pixels to differ substantially from the class average. For illustration, the majority of gray pixels in $\overline{X}$ are white for those $D_B$, but there might still be a considerable number of pixels that differ strongly. The overall shape might even indicate a different class, i.e., some T-Shirts appear as shirts. This is also aligned with the observation that these samples' standard deviation is high.

\begin{figure}[!htb]
\centering{
\centerline{\includegraphics[width=0.4\textwidth]{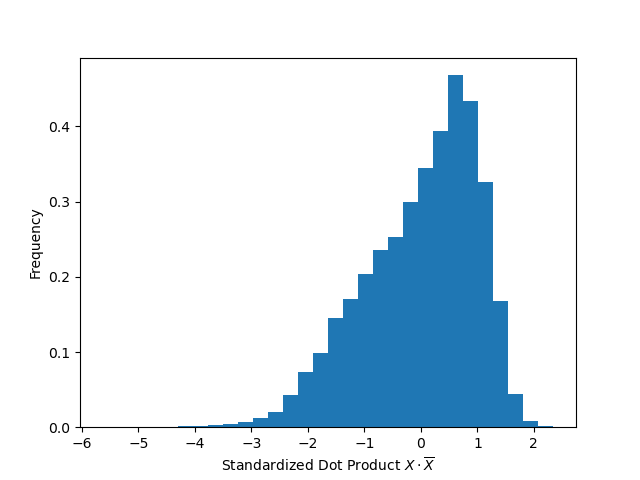}} 
\vspace{-12pt}
\caption{Distribution of the dot product of a sample and the mean, i.e. $X\cdot \overline{X}$. Dot products are standardized for each class.} \label{fig:diFa}}
\end{figure}

Fitting to incorrect samples with large deviation from the mean distorts the well-visible ``prototypes'' shown in  Figure \ref{fig:weiFa}. Distorted ``prototypes'' show more contrast and larger differences between adjacent weights. A weight is either very small (black) or very large (white), and neighboring weights often have different signs. To understand the process that leads to higher reconstruction loss, we can view the second phase of learning as fitting to a few high-loss samples.
The pixel average for easy and hard samples is about the same. In this case, the weight will increase in magnitude, e.g., in Figure \ref{fig:weiFa} a bright pixel becomes slightly brighter, and a dark pixel gets darker. This helps to more reliably classify easy samples, and it improves the loss of hard samples.
If the average of samples, i.e., $D_A$, differs a lot for a specific pixel from high loss samples $D_S$, in the phase where the loss of $D_A$ is low but still high for $D_S$, the weight is changed considerably, and a black pixel might become white and vice versa.

\paragraph{Multiple layers and networks:}
Figure \ref{fig:metMuFa} shows the outputs for the last and second last layer for multiple networks for the FashionMNIST dataset. (Additional datasets are in the Appendix). For the last layer, all networks behave qualitatively identically. For the second last layer, the overall pattern remains. Generally, for layers closer to the input, it gets weaker. 

\begin{figure}
  \centering
  \subfloat[Last Layer]{\includegraphics[width=0.35\textwidth]{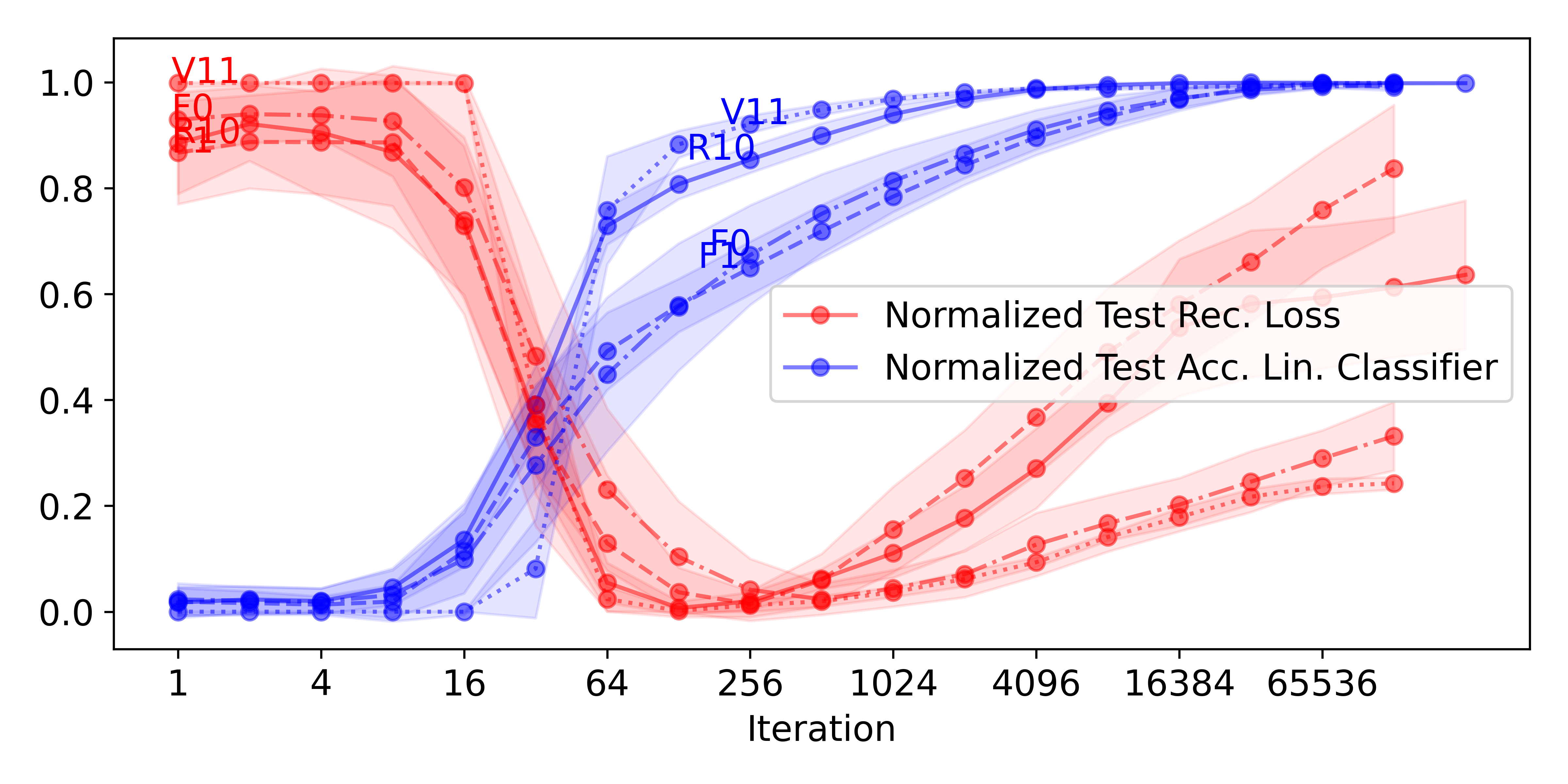} \label{fig:a}} \\
  \vspace{-12pt}
  \subfloat[Second last layer]{\includegraphics[width=0.35\textwidth]{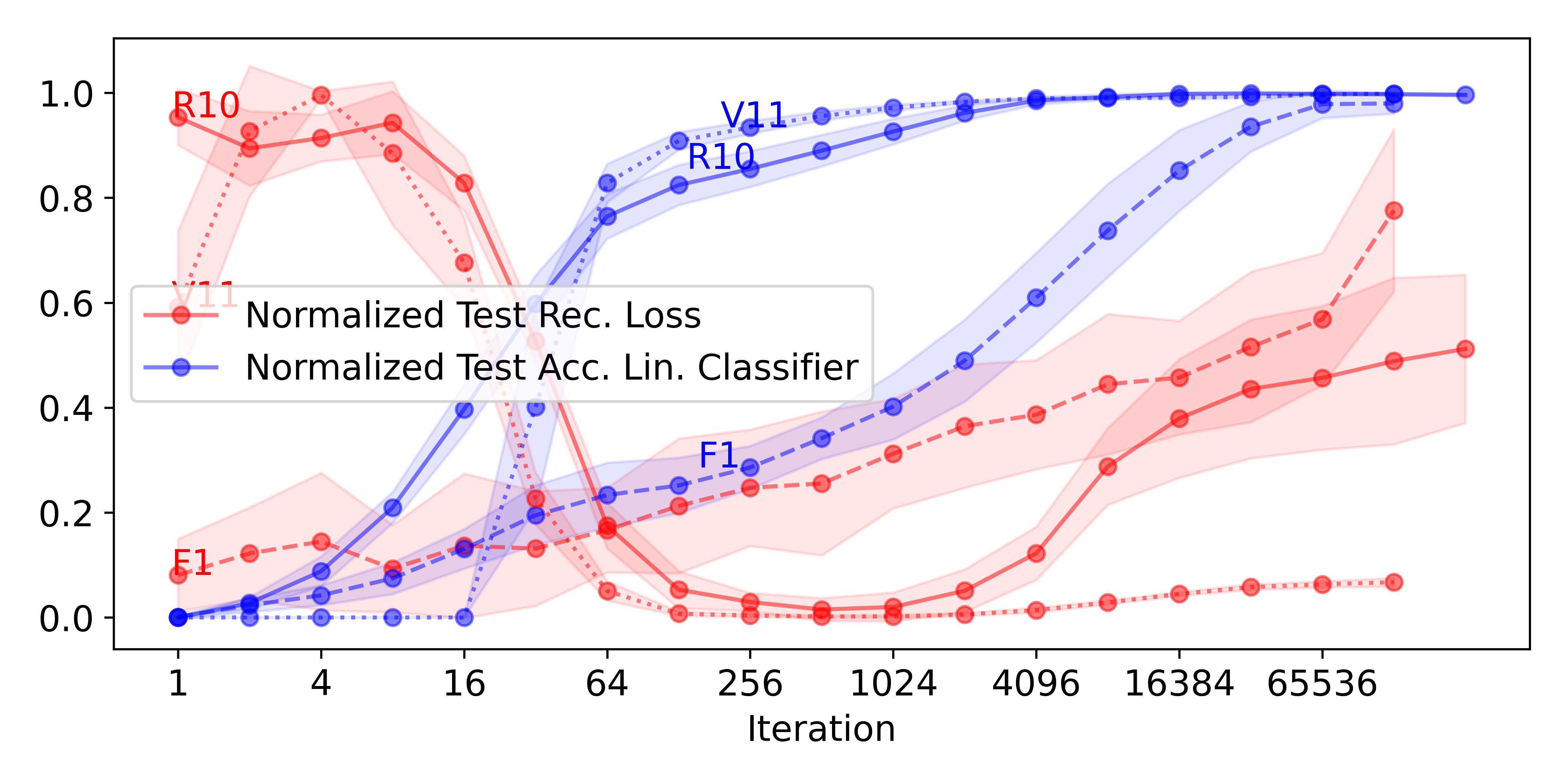} \label{fig:b}} \\
  \vspace{-6pt}
  \caption{Normalized accuracy and reconstruction loss for multiple classifiers for the FashionMNIST dataset. Other datasets are in the Appendix.} \label{fig:metMuFa}
  \vspace{-12pt}
\end{figure}

\section{Theoretical analysis}
We follow standard complexity analysis from computer science deriving bounds regarding the number of inputs $n$ assuming $n$ is large, allowing to discard lower order terms in $n$. 

\subsection{Model and Definitions}
We describe our dataset, network, loss and optimization as well as how inputs are reconstructed from layer activations. 
\paragraph{Data:}
Our data is defined to have the following data characteristics based on our empirical analysis: (i) Most samples (of a class) are similar. Still, there are a few samples that differ significantly from the majority, e.g., see Figure \ref{fig:diFa}. In particular, classification loss (at least early in training) differs significantly for a few samples, i.e., it might even increase as shown for multiple classes in Figure \ref{fig:loClFa}. (ii) Most samples can be classified (correctly) using a subset of all available attributes. This holds in particular for correlated inputs, e.g., down-sampled images still allow to train well-performing classifiers. (iii) For a class, attributes of inputs have different means and variations, e.g., making them more or less sensitive to (additive) noise. (iv) Occam's razor principle: The model should be as simple as possible to allow for rigorous analysis. \\
We focus on \emph{binary classification} where each input has two attributes. We consider a labeled dataset $D={(X,Y)}$ consisting of pairs $(X,Y)$ with input $X=(X_0,X_1)$ and label $Y \in \{0,1\}$. We denote $n =|D|$ as the number of samples. We denote $C^y=\{X |(X,Y) \in D \wedge Y=y\}$ as all inputs of class $y$. We assume balanced classes, i.e., $|C^y|=|C^{y'}|$ for arbitrary $y,y'$.

We proclaim that samples $C^y$ of class $y$ can be split into a \emph{big subset} $C^y_{k}\subset C^y$ and a \emph{small subset} $C^y_{1-k}\subset C^y$ with $k\in ]0,1[$. We define points $X=(X_0,X_1)$ as follows:

  \begin{equation}
  \footnotesize{
    X= \label{def:dat}
    \begin{cases}
      X_0=(1-2\cdot y), & \\
      X_1 =(1-2\cdot y)\cdot b+\epsilon & \text{if}\ X \in  C^y_{k}\\\
      X_0=-(1-2\cdot y),X_1=0, & \text{if}\ X \in  C^y_{1-k}\\
    \end{cases}
    }
  \end{equation}
  
, where $\epsilon \sim U(-\sigma,\sigma)$ resembles Uniform noise. 

\begin{assumption}[Data Parameters] \label{ass:initbs}
We set $b\in[0,1/\log(n)^2]$, $\sigma\in [1/\log(n)^3,1/\log(n)^4]$, $k=1/\sqrt{n}$
\end{assumption} 
Thus, a sample $X$ can be classified (correctly) using either $X_0$ or $X_1$ or both. However, attribute $X_1$ has lower (average) magnitude than $X_0$. Thus, relying (only) on $X_1$ might unavoidably lead to misclassifications in the presence of large additive noise.  Other choices for $X_0$ and $X_1$ are possible, if they ensure assumption (i). Generalization is discussed in Section \ref{sec:gen}.




\paragraph{Network: }
We use a network with a single layer with weights $W^{(t)}=(w^{(t)}_0,w^{(t)}_1)$ at iteration $t$ during training leading to a scalar output $o$. We omit the superscript $(t)$ if there are no ambiguities. 

\begin{definition}[Network Output] \label{def:out}
$o:=o(X):= W^{(t)}\cdot X=w^{(t)}_0\cdot X_0+w^{(t)}_1\cdot X_1$.
\end{definition}

\begin{assumption}[Weight Initialization] \label{ass:initw}
$|w^{(0)}_0|<1$ and $1/\log(n)<|w^{(0)}_1|<1$.
\end{assumption}
Initialization schemes \cite{he15,sch22} typically initialize weights using random values with mean zero sampled either from the uniform or Gaussian distribution with standard deviation mostly depending on the in- and out-fan. Our assumption covers all values from common uniform initialization schemes. The lower bound for $|w^{(0)}_1|$ eliminates corner cases in the analysis discussed in Section \ref{sec:gen}. 

\paragraph{Loss and optimization: }
We use the logistic function to compute class probabilities given the output $o$: 
$q=q(y=1|X)=\frac{1}{1+\exp(-o)}$, $q(y=0|X)=1-q(y=1|X)$.
The loss for a sample $X$ is given by $L(X)=y\cdot\log(q(o))+(1-y)\cdot\log(1-q(o))$. We perform gradient descent. A value of weight $w^{(i)}_j$ at iteration $i$ of gradient descent is defined as 

\begin{align}
w^{(i+1)}_j&=w^{(i)}_j-\frac{\lambda}{|D|} \sum_{X \in D}  \nabla_{w_j} L(X) \label{eq:grad}  
\end{align}
We assume a fixed learning rate of $\lambda=1/2$.

\paragraph{Reconstruction: }
To compute the reconstruction error for a sample $X$ using its output $o(X)$, we fit a linear reconstruction function $g_j$ for each input feature $X_j$ using layer activations of all training data, i.e., $\{o(X)|X \in D\}$.


\begin{definition}[Reconstruction Function] \label{def:recF}
$g_j(o):=v^0_j\cdot o+v^1_j$
\end{definition}

The reconstruction loss for an input attribute $X_j$ of an input $X=(X_0,X_1)$ is given by 
\begin{definition}[Reconstruction Loss] \label{def:recLoss}
$R_j(X):=||X_j-g_j(o(X))||^2$.
\end{definition}

\subsubsection{Prerequisites}

Before our main analysis, we derive a few basic results.
The derivative $dL/dw_i$ of the loss with respect to network parameters is
$dL/dw_i=dL/do \cdot do/dw_i=(q(o)-y)\cdot X_i$ (see e.g. Section 5.10 in \cite{jur22})

For class $y=0$ and for $X \in C^0_{k}$ using $X_0=1,X_1=b+\epsilon$ (Def. \ref{def:dat}) and for $X \in C^0_{1-k}$ with $X_0=-1,X_1=0$ we get:
\begin{align*} 
\frac{dL(X \in C^0_{k})}{dw_0}&= 1/(1+e^{-w_0-w_1(b+\epsilon)}) \\ 
\frac{dL(X \in C^0_{k})}{dw_1}&= b/(1+e^{-w_0-w_1(b+\epsilon)}) \\ 
\frac{dL(X \in C^0_{1-k})}{dw_0}&= -1/(1+e^{w_0}) \\ 
\frac{dL(X \in C^0_{1-k})}{dw_1}&=  0 
\end{align*} 


Thus, the sum of the derivatives for all samples $X\in D$ is given due to symmetry with respect to $y\in \{0,1\}$ by:
\eq{
dL/dw_i &:= \sum_{X \in D} dL(X)/dw_i= 2 \sum_{X \in C^0} dL(X)/dw_i
}
For notational ease, we subsume the factor 2 it in the learning rate, i.e. using $\lambda'=1$ instead of $\lambda=1/2$ (Eq. \ref{eq:grad}).
We get:
\eq{
dL/dw_0 &:= k/(1+e^{-w_0-w_1(b+\epsilon)}) - (1-k)/(1+e^{w_0}) \label{eq:dw0}\\
dL/dw_1 &:= kb/(1+e^{-w_0-w_1(b+\epsilon)}) \label{eq:dw1}
}

\begin{lemma}\label{le:der} 
It holds that $dL/dw_1<b$ and $dL/dw_0>\exp(-2m)/8$ with $m:=\max(|w^{(t)}_0|,|w^{(t)}_1|)$ and $m<\log(n)/4-4$.
\end{lemma}

\begin{proof}
For the logistic activation holds that  $1\geq 1/(1+e^{-o})>e^o/4$ for $o\leq 1$.
\footnotesize{
\begin{align*}
dL/dw_0&=\frac{k}{1+e^{-w_0-w_1(b+\epsilon)}} - \frac{1-k}{1+e^{w_0}} (\text{ Eq.  \ref{eq:dw0}})  \\
&>ke^{-2m}/4-(1-k)  \text{ with } m:=\max(|w^{(t)}_0|,|w^{(t)}_1|) \\ &\phantom{abc} \text{ since } |b+\epsilon|<1\\
&>(1-1/\sqrt{n})\exp(-2c)/4-1/\sqrt{n}\\
&>\exp(-2m)/4-2/\sqrt{n} \text{ since } m<\log(n)/4-4\\
&>\exp(-2m)/8
\end{align*}}

We use Eq. \ref{eq:dw1} and $k<1$ (Ass. \ref{ass:initbs}).
\begin{align*}
dL/dw_1&=\frac{kb}{1+e^{-w_0-w_1(b+\epsilon)}}<b
\end{align*}

\end{proof}


Next, we bound the expected reconstruction error for $E[R_j(X)]$.  
Consider a sample $X \in (C^0_{k} \cup C^1_{k})$ from a big subset of any of the two class. The reconstruction function $g_j$ (Def. \ref{def:recF}) becomes for $v^0_0=1/(w_0+w_1b)$, $v^0_1=b/(w_0+w_1b)$ and $v^1_j=0$.
\begin{align} 
g_0(o)&=v^0_0\cdot o+v^1_0=o/w_0=o\cdot 1/(w_0+w_1b) \label{eq:g0} \\ 
g_1(o)&=o/w_0=o\cdot b/(w_0+w_1b)  \label{eq:g1}
\end{align}
The expected reconstruction error for $E[R_0(X)]$ for $X_0$ and $X \in C^0_{k}$ can be estimated using $X_0=1$,$X_1=b+\epsilon$ (Def. \ref{def:dat}) yielding $o:=w_0+w_1(b+\epsilon)$ (Def. \ref{def:out}), and therefore using Equation \ref{eq:g0} $g_0(o)=1+w_1\epsilon/(w_0+w_1b)$. In turn, this gives using  Def. \ref{def:recLoss} and linearity of expectation, i.e., $E[aX+b]=aE[X]+b$ for constants $a,b$:

\begin{align} 
E[R_0(X)]&=E[(X_0-g_0(o))^2] \nonumber\\
&=E[(1-(1+\frac{w_1\epsilon}{w_0+w_1b}))^2] \nonumber\\
&=w_1^2/(w_0+w_1b)^2E[\epsilon^2]\nonumber\\
&=(w_1/(w_0+w_1b))^2\sigma^2 \label{eq:rec0} \\
E[R_1(X)]&=(w_1b/(w_0+w_1b))^2\sigma^2 \label{eq:rec1}
\end{align} 


Thus, the reconstruction error is optimal, i.e. zero, if no noise is present ($\sigma=0$). 

This leaves us to bound the reconstruction error $R_j$ for the smaller subset $X \in (C^0_{1-k} \cup C^1_{1-k})$.
We have for $X \in C^0_{1-k}$ that $X_0=-1,X_1=0$ (Def. \ref{def:dat}) and $o=-w_0$ (Def. \ref{def:out}) and using Eq. \ref{eq:g0}: $g_1(o):=-w_0/(w_0+w_1b)$
Thus, $E[(X_0-g_0(o))^2]= (-1+w_0/(w_0+w_1b))^2$
For $X \in C^0_{1-k}$ holds in the same manner $E[(X_1-g_1(o))^2]= (1-w_0/(w_0+w_1b))^2$. 

The error is not optimal.  Since the sets $C^y_{1-k}$ are very small, the total aggregated error for $X \in (C^0_{1-k} \cup C^1_{1-k})$ is small. It can be mostly neglected compared to that of $C^y_{k}$.


The reconstruction errors for subsets of class $y=1$, i.e. $X \in C^1$, are identical to those of class $C^0_{k}$ due to symmetry.

For noise $\epsilon \sim U(-\sigma,\sigma)$ with large variance $\sigma^2$, reconstruction of $X_i$ using a linear function $g_i(o)$ cannot leverage the information in $o$. It is better to neglect the computed output $o(X)$ when reconstructing $X_i$ from $o$ and simply use the mean $\overline{X_i}$, e.g. $g_0(X)=\overline{X_0}=0$, this gives error: $\sum_{X \in D} R_j(X) =\sum_{X \in D} (X_0-\overline{X_0})^2=\sum_{X \in D} X_0^2=|D|$
For $g_1(X)=\overline{X_1}=0$, this gives error: \\
$\sum_{X \in D} R_j(X) =\sum_{X \in D} (X_1-\overline{X_1})^2$\\
$=\sum_{X \in D} X_1^2 = (1-k)n\cdot b^2$ 

Let us consider the reconstruction loss for $w_0$, i.e., $w_1/(w_0+w_1b)^2\cdot \sigma^2$ (Equation \ref{eq:rec0}). It increases with $t$ as long as $\frac{w^{(t)}_1}{(w^{(t)}_0+w^{(t)}_1b)^2}>\frac{w^{(t-1)}_1}{(w^{(t-1)}_0+w^{(t-1)}_1b)^2}$. The inequality holds if the relative increase of the nominator is larger than that of the denominator, i.e.,
\eq{ 
\frac{dL/dw_1}{w^{(t)}_1} >  \big(\frac{dL/dw_0+ dL/dw_1\cdot b}{w^{(t)}_0+w^{(t)}_1b}\big)^2 \label{eq:recerr}
}

\begin{lemma}\label{le:dw0}
$\frac{dL/dw_0}{ w^{(t)}_0}  > \frac{1}{\sqrt{\log (n)}} \text { for }t<\log \log n/8$
\end{lemma}

\begin{proof}
Using Lemma \ref{le:der}  $|dL/dw_0|>\exp(-2m)/8> 1/\log(n)^{1/4}/8> 1/\log(n)^{1/3}$.

Also, $w^{(t)}_0<1+t $ since $|w^{(0)}_0|<1$ and $|dL/dw_0|\leq 1$ (Eq. \ref{eq:dw0}). Thus,
$\frac{dL/dw_0}{ w^{(t)}_0} > \frac{1/\log(n)^{1/3}}{1+\log \log n/8}>1/\sqrt{\log(n)}$
  
\end{proof}

\begin{lemma}\label{le:w1}
For $t<\log n/2$, $|w^{(t)}_1|/2\geq |w^{(0)}_1|/2\geq 1/(2\log(n))$
\end{lemma}
\begin{proof}
Using Def. \ref{def:dat} and Eq. \ref{eq:dw1} it follows that $0<dL/dw_1<b<1/\log^2 n$.
We use that $dL/dw_1<b$ for any $t$ (Lemma \ref{le:der} giving for $t<\log n/2$:
\begin{align*}
|w^{(t)}_1|&\geq |w^{(0)}_1|-tb \\
&>1/\log(n)-t/\log^2(n) (Ass. \ref{ass:initw}, \ref{ass:initbs})\\
&>1/\log(n)-t/\log^2(n)\\
&>1/\log(n)/2\\
&>|w^{(0)}_1|/2
\end{align*}
We also have that $|w^{(0)}_1|\geq 1/\log(n)$ due to Ass. \ref{ass:initw}.
\end{proof}

\begin{lemma}\label{le:dw1}
$\frac{dL/dw_1}{ w^{(t)}_1}  < 1/\log (n) \text{ for }t<\log n/2$
\end{lemma}

\begin{proof}
Using Def. \ref{def:dat} and Eq. \ref{eq:dw1} it follows that $0<dL/dw_1<b\leq 1/\log^2 n$.
Using Lemma \ref{le:w1} we have that  $\frac{dL/dw_1}{ w^{(t)}_1}  < 1/\log^2 n\cdot (2\log(n))< 1/\log(n)$ 
\end{proof}

\subsection{Learning phases}


\begin{theorem} \label{thm:init}
After constant iterations $t\leq 1500$, it holds that $w^{(t)}_0\geq 2$. If $w^{(t')}_0>1\sqrt{\log(n)}$  the reconstruction error decreases for any $t'<1500$.
\end{theorem}

\begin{proof}
By Ass. \ref{ass:initw} we have $|w^{(0)}_0|<1$ and $1/\log(N)<|w^{(0)}_1|<1$. Thus, after $t\geq 1500$ iterations for weights holds $|w^{(t)}_0|<1+t$, since  $|dL/dw_0|\leq 1$ (Eq. \ref{eq:dw0}) and  $w^{(t)}_1<1+tb<2$ since $|w^{(0)}_1|<1$ and the change $|dL/dw_1|\leq b$ (Eq. \ref{eq:dw1}).

We assume that $w^{(0)}_0=-1$, since this requires the largest changes, i.e. most iterations, to reach $w^{(t)}_0\geq 2$. 
Thus, for $-1<w^{(t)}_0<2$ and $w^{(t)}_1<1+tb<2$ using Lemma \ref{le:der} with $m\leq 2$ we get that $dL/dw_0>\exp(-4)/8>0.002$. We upper bound $w^{(t)}_1$ using Lemma \ref{le:w1}.


For $|w^{(t)}_0+w^{(t)}_1b|>1/\sqrt{\log(n)}$, which holds for any $w^{(t)}_0>2/\sqrt{\log(n)}$ since $|w^{(t)}_1b|<1/\log(n)^2$ and $dL/dw_0>0$.

The reconstruction error decreases for $|w^{(t)}_0|>1/\sqrt{\log(n)}$, i.e. analogous to Eq. \ref{eq:recerr} holds that 
$dL/dw_1  / w^{(t)}_1 <  (dL/dw_0+bdL/dw_1)^2  / (w^{(t)}_0+w^{(t)}_1b)^2$ plugging in prior upper bounds for $|dL/dw_1|\leq b$ and lower bounds on $dL/dw_0+dL/dw_1>0.002+b>0.002$ and $w^{(t)}_1$ (Lemma \ref{le:w1}). Note that $|w^{(t)}_0+w^{(t)}_1b|>1/\sqrt{\log(n)}$ for $w^{(t)}_0>1/\sqrt{\log(n)}$ 
$w^{(t)}_1<1+tb<2$

Thus, to shift $w_0$ by $3$ requires at most $3/0.002=1500$ iterations. Using Lemma \ref{le:dw0} $dL/dw_0>0$ for $-1<w^{(t)}_0<2$. Furthermore, 
\end{proof}
Technically, in case $w^{(t)}_0+w^{(t)}_1b\approx 0$ the reconstruction error might also increase. However, since changes to $w_0$ are large and $w_1b$ is roughly constant, this might not necessarily happen if  $w^{(t)}_0+w^{(t)}_1b$ changes from being negative to positive. It depends on the exact initialization of $w_0$.

Next, we investigate the learning behavior after the first iterations, i.e. once $w^{(t)}_0>2$.
\begin{theorem}
For iterations $t\in[1500,\log(n) /8]$ the reconstruction error will decrease. For $t>t_0$ for some $t_0 > \log(n)/8$ it will increase again.
\end{theorem}

\begin{proof}
Using Theorem \ref{thm:init} at $c=1500$ iterations $w^{(c)}_0>1$ and the reconstruction error decreases.
We proceed by showing that it still decreases at iteration $t=\log \log n/8$. 


Using Lemma \ref{le:der} $dL/dw_0>\exp(-2m)/8$ for $m<\log (n)/4-4$, i.e. $w^{(t)}_0$ increases for $t \in [0,\log n/8]$ and, still, $w^{(t)}_0>2$.

In contrast $|w^{(t)}_1|<1$, changes by at most $b$ after $t\cdot b <1\log N$ thus $|w^{(t)}_1|<2$.
Thus, $m=\max(w^{(t)}_0,w^{(t)}_1)=w^{(t)}_0$
Let us bound the number of iterations $i$ until $w^{(t')}_0$ changes by 1, i.e.,
$w^{(t'+i)}_0\geq 1+w^{(t')}_0$
For any $0<j<i$ holds
$dL/dw_0>\exp(-2w^{(t'+j)}_0)/8>\exp(-2w^{(t'+i)}_0)/8 >\exp(-2w^{(t')}_0-2)/8$ 
Thus, to compute the number of iterations $i$ to change $w^{(t)}_0$ by 1 we use:
\begin{align*}
i\exp(-2w^{(t)}_0-2)/8=1\\ 
i=8\exp(2w^{(t)}_0+2)
\end{align*}
Thus, given a total of $t\leq \log(n)/8$ iterations, we have for the final weight $w'$:
\begin{align*}
&\sum_{j<w'} 8\exp(2j+2) = t\\
&8\exp(2w'+3) \geq t \text{ Using } \sum_{u<x} 2^u \leq 2^{u+1}\\
&2w'+3+\log 8 \geq \log(t)\\
&2w'+6 \geq \log(t)\\
&w' \geq \log(t)/2-6  
\end{align*}
Using $w^{(0)}_0\geq -1$, we get that $w^{(t)}_0\geq -\log(t)/2-7 = -\log \log(n)/2-9$ for $t=\log(n)/8$.

Comparing the two terms $k/(1+exp(-w_0-w_1b))$, $(1-k)/(1+exp(w_0))$ for $dL/dw_0$ from Lemma \ref{le:dw0}, it can be seen that the decrease of $w_0$ must end, i.e. $dL/dw_0$ changes signs, since $w^(t)_1$ can only increase and the first term $k/(1+exp(-w_0-w_1b))$ tends to 0, while the second $(1-k)/(1+exp(w_0))$ tends to $1-k$.
\end{proof}

\subsection{Discussion} \label{sec:gen}
A few samples can cause a large shift of attribute weights, undoing changes performed in reducing classification loss for the majority of samples. This can also reduce robustness of classification due to additive noise if the classifier relies on attributes being not very discriminative.  A shift of weights also impacts reconstruction loss, i.e., if noise of attributes gets amplified due to multiplication with larger weights reconstruction errors increase. 

Using the bounds for the derivatives of $w_0$ and $w_1$ (Equation \ref{eq:dw0} and \ref{eq:dw1}, it also becomes apparent that during initial training, i.e. as long as $w_0$ is small, we essentially add in each iteration the mean vector of the large subset $C^1_{k}$ to the existing weights. Thereby, weights $W=(w_0,w_1)$ point more and more towards the class mean $\overline{X}$ for $X \in C^1_{k}$.

We have assumed a lower bound of $1/\log(n)$ for the absolute value of the initialized weight $|w^{(0)}_1|$. In our analysis, it was needed to ensure that  $\frac{dL/dw1} /w1$ remains small (essentially constant) while $w_0$ changes.
If this does not hold, i.e., $|w^{(t)}_1|\approx 0$, the reconstruction error could be 0 (assuming $w_0\neq 0$). Thus, in principle, in a phase where we claim that the reconstruction error increases, it might decrease (at least for some iterations) under the condition that $w_1$ was initialized with a value of very small magnitude.

In our analysis, we were unspecific of what happens if $w^{0}_0<1/\sqrt{n}$. We have shown in Theorem \ref{thm:init} that $w_0$ changes signs if it is negative initially, but not what happens for the reconstruction error. In our reconstruction function, the coefficient $c^0_0:=1/(w_0+w_1\cdot b)$ might become unbounded for $w^{0}_0<1/\sqrt{n}$. Note that $w_0+w_1\cdot b=0$ implies that the output is zero for inputs of the large subsets of both classes. Thus, the outputs are of no value for reconstruction, and we might use the mean for reconstruction, i.e., use $c^0_0=0$. It also means that it might be possible that early the reconstruction error initially increases before decreasing but whether this happens depends on the exact value of $w^{0}_0$, i.e., it is not sure to happen just because $w^{0}_0<1/\sqrt{n}$.




\section{Related Work}
\cite{tish15} proposed both the information bottleneck \cite{tish00} and its usage for analysis of deep learning. It suggests a principled way to ``find a maximally compressed mapping of the input variable that preserves as much as possible the information on the output variable''\cite{tish15}. To this end, they view layers $h_i$ of a network as a Markov chain for which holds given $i\geq j$ using the data processing inequality: $$I(Y;X) \geq I(Y;h_j) \geq I(Y;h_i) \geq I(Y;\hat{Y})$$ 
They view learning as the process that maximizes $I(Y;h_i)$ while minimizing $I(h_{i-1};h_i)$, where the latter can be interpreted as the minimal description length of the layer. In our view, we at least on a qualitative level agree on the former, but we do not see minimizing the description length as a goal of learning. In our perspective and also in our framework, it can be a consequence of the first objective, i.e. to discriminate among classes, and existing learning algorithms, i.e., gradient descent. From a generalization perspective, it seems preferable to cling onto even the smallest bit of information of the input $X$, even if its highly redundant and, as long as it \emph{could} be useful for classification. This statement is also supported by \cite{saxe19} who show that compression is not necessary for generalization behavior and that fitting and compression happen in parallel rather than sequentially. A recent review \cite{gei21} also concludes that the absence of compression is more likely to hold. This is aligned with our findings, since in our analysis, ``information loss'' on the input is a consequence of weighing input attributes too much that are sensitive to additive noise. In contrast to our work, their analysis is within the IB framework. Still, it remedies an assumption of \cite{tish15} namely \cite{saxe19} investigates different non-linearities, i.e., the more common ReLU activations rather than sigmoid activations. Recently, \cite{lor21} argues that compression is only observed consistently in the output layer. The IB framework has also been used to show that neural networks must lose information \cite{liu20} irrespective of the data it is trained on. From our perspective the alleged information loss measured in terms of the reconstruction capability could be minimal at best. In particular, it is evident that at least initially reconstruction is almost perfect for wide networks following theory on random projection, i.e., the Johnson-Lindenstrauss Lemma\cite{joh84} proved that random projections allow embedding $n$ points into an $O(\log n/\epsilon^2)$ dimensional space while preserving distances within a factor of $1\pm \epsilon$. This bound is also tight according to \cite{lar17} and easily be extended to cases where we apply non-linearities, i.e., $ReLU$. \\
\cite{gei21} also discussed the idea of geometric compression based on prior works on IB analysis. However, the literature was inconclusive according to \cite{gei21} on whether compression occurs due to scaling or clustering. Our analysis is inherently geometry (rather than information) focused, i.e., we measure learning based on how well classes can be separated using linear separators in learnt latent spaces. That is, our work favors class-specific clustering as set forth briefly in \cite{gol18}, but we derive it not using the IB framework.\\
As the IB has also been applied to other types of tasks, i.e., autoencoding \cite{tap20}, we believe that our approach might also be extended to such tasks.\\

The idea to reconstruct inputs from layer activations has been outlined in the context of XAI \cite{sch21cla}. The idea is to compare reconstructions using a decoder with original inputs to assess what information (or concepts) are ``maintained'' in a model. Our work also touches upon linear decoders that have been studied extensively, e.g., \cite{kun19}. It also estimates reconstruction errors from noisy inputs $X_i+\epsilon$ \cite{car09}.

\section{Conclusions}
Theory of deep learning is limited. This work focused on a very pressing problem, i.e., understanding the learning process. To this end, it rigorously analyzed a simple dataset modeling many observations of common datasets. Our results highlight that few samples are likely to profoundly impact weights in later stages of the training, potentially compromising classifier robustness.

\bibliography{refs}
\bibliographystyle{icml2022}

\end{document}